\documentclass[a4paper,twoside]{article}

\usepackage{epsfig}
\usepackage{calc}
\usepackage{amssymb}
\usepackage{amstext}
\usepackage{amsmath}
\usepackage{amsthm}
\usepackage{multicol}
\usepackage{pslatex}
\usepackage{apalike}
\usepackage{SCITEPRESS}    
\usepackage[]{algorithm2e}
\usepackage{multirow}

\usepackage{graphicx}
\usepackage{subcaption}

\newtheorem{definition}{Definition}
\newtheorem{theorem}{Theorem}

\newcommand{\cmmnt}[1]{\ignorespaces}

\begin{document}

\title{Area Protection in Adversarial Path-Finding Scenarios with Multiple Mobile Agents on Graphs
\subtitle{a theoretical and experimental study of target-allocation strategies for defense coordination} }

\author{\authorname{Marika Ivanov\'{a}\sup{1}, Pavel Surynek\sup{2}}
\affiliation{\sup{1}University of Bergen, Realfagbygget, All\'{e}gaten 41, 5020 Bergen, Norway}
\affiliation{\sup{2}National Institute of Advanced Industrial Science and Technology (AIST), 2-3-26, Aomi, Koto-ku, Tokyo 135-0064, Japan }
\email{marika.ivanova@uib.no, pavel.surynek@aist.go.jp}
}

\keywords{graph-based path-finding, area protection, area invasion, asymmetric goals, mobile agents, agent navigation, defensive strategies, adversarial planning}

\abstract{We address a problem of area protection in graph-based scenarios with multiple agents. The problem consists of two adversarial teams of agents that move in an undirected graph shared by both teams. Agents are placed in vertices of the graph; at most one agent can occupy a vertex; and they can move into adjacent vertices in a conflict free way. Teams have asymmetric goals: the aim of one team - {\em attackers} - is to invade into given area while the aim of the opponent team - {\em defenders} - is to protect the area from being entered by attackers by occupying selected vertices.
We study strategies for allocating vertices to be occupied by the team of defenders to block attacking agents. We show that the decision version of the problem of area protection is PSPACE-hard under the assumption that agents can allocate their target vertices multiple times.
Further we develop various on-line vertex-allocation strategies for the defender team in a simplified variant of the problem with single stage vertex allocation and evaluated their performance in multiple benchmarks. The success of a strategy is heavily dependent on the type of the instance, and so one of the contributions of this work is that we identify suitable vertex-allocation strategies for diverse instance types. In particular, we introduce a simulation-based method that identifies and tries to capture bottlenecks in the graph, that are frequently used by the attackers. Our experimental evaluation suggests that this method often allows a successful defense even in instances where the attackers significantly outnumber the defenders.     }

\onecolumn \maketitle \normalsize \vfill

\section{\uppercase{Introduction}}
\label{sec:introduction}
\noindent
\begin{figure}[!h]
  \centering
   {\epsfig{file = 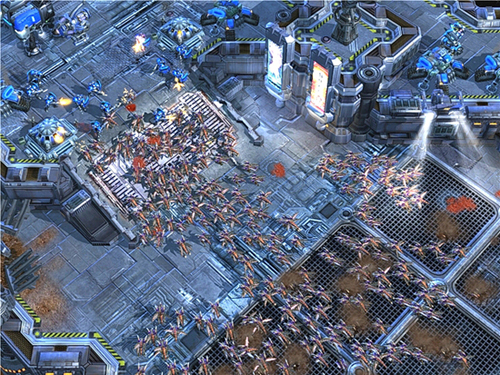, width = 5.0cm}}
  \caption{A motivation example of an adversarial path-finding scenario}
  \label{fig:starcraft}
 \end{figure}
In this work we study an {\em Area Protection Problem} (APP) with multiple mobile agents moving in a conflict free way. APP can be regarded as a modification of known problem of {\em Adversarial Cooperative Path Finding} (ACPF) \cite{IvanovaS14} where two teams of agents compete in reaching their target positions. Unlike ACPF, where the goals of teams of agents are symmetric, the adversarial teams in APP have different objectives. The first team of {\em attackers} contains agents whose goal is to reach a pre-defined target location in the area being protected by the second team of {\em defenders}. Each attacker has a unique target in the protected area and each target is assigned to exactly one attacker. The opponent team of defenders tries to prevent the attackers from reaching their targets by occupying selected locations so that they cannot be passed by attackers.

The common feature of APP and ACPF is that once a location is occupied by an agent it cannot be entered by another agent until it is first vacated by the agent which occupies it (opposing agent cannot push it out). This property is exploited both in competition for reaching goals in ACPF so that a team of agents may try to slow down the opponent by occupation of certain locations as well as in APP where this is a key tool for the team of defenders.

APP has many real-life motivations from the domains of access denial operations both in civil and military sector, robotics with adversarial teams of robots or other type of penetrators \cite{DBLP:journals/jair/AgmonKK11}, and computer games (see Figure \ref{fig:starcraft}).

Our contribution consists in analysis of computational complexity of APP. Particularly we show that APP is PSPACE-hard. Next we contribute by suggesting several on-line solving algorithms for the defender team that allocate suitable vertices to be occupied so that attacker agents cannot pass into the protected area. We identified suitable vertex allocation strategies for diverse types of APP instances and tested them thoroughly.

\subsection{Related Works}
Movements of agents at low reactive level are assumed to be planned by some {\em cooperative path-finding - CPF} ({\em multi-agent path-finding} - MAPF) \cite{Silver05,Ryan08,WangB11} algorithm where agents of own team cooperate while opposing agents are considered as obstacles. In CPF the task is to plan movement of agents so that each agent reaches its unique target in a conflict free manner.

There exist multiple CPF algorithms both complete and incomplete as well as optimal and sub-optimal under various objective functions. It is known that any known optimization version of CPF is an NP-hard problem \cite{RatnerW90,YuL13}. Many efficient optimal algorithms that introduce advanced search space transformations like {\em CBS} \cite{SharonSFS15} or {\em ICTS} \cite{SharonSGF13} and compilation based methods that cast CPF to a different formalism \cite{SurynekFSB16} have been introduced recently. However scalability of these optimal algorithms is limited which makes them unsuitable in our reactive setup where massive numbers of agents is expected.

Suboptimal CPF algorithms include rule-based polynomial time methods and search-based algorithms. Rule-based algorithms like {\em BIBOX} \cite{Surynek14} and {\em Push-and-Swap} \cite{LunaB11,WildeMW14} guarantee finding solution in polynomial time. These algorithms scale up well for large number of agents and large graphs however solutions generated by them are usually very far from the optimum with respect to any common objective.

A good compromise between optimal and rule-based algorithms is represented by suboptimal/incomplete search based methods which are derived from the standard {\em A*} algorithm. These methods include {\em LRA*}, {\em CA*}, {\em HCA*}, and {\em WHCA*} \cite{Silver05}. They provide solutions where individual paths of agents tend to be close to respective shortest paths connecting agents' locations and their targets. Conflict avoidance among agents is implemented via a so called reservation table in case of {\em CA*}, {\em HCA*}, and {\em WHCA*} while {\em LRA*} relies on replanning whenever a conflict occurs. Since our setting in APP is inherently suitable for a replanning algorithm {\em LRA*} is a candidate for underlying CPF algorithm for APP. Moreover {\em LRA*} is scalable for large number of agents.

Aside from CPF algorithms, systems with mobile agents that act in the adversarial manner represent another related area. These studies often focus on patrolling strategies that are robust with respect to various attackers trying to penetrate through the patrol path \cite{DBLP:journals/amai/ElmaliachAK09}. Theoretical works related to APP also include studies on {\em pursuit evasion} \cite{DBLP:journals/trob/VidalSKSS02} or {\em predator-prey} \cite{DBLP:conf/ijcai/HaynesS95} problems. The major difference between these works and the concept of APP is that we consider relatively higher number of agents and our agents are more limited in their abilities.

\subsection{The Target Allocation Problem}

It this paper we specifically focus on a sub-problem called {\em target allocation problem}.
The defenders are initially not assigned to any targets and don't have any information about the intended target of any attacker. However, the defenders have a full knowledge of all target locations in the protected area. The task in this setting is to allocate each defender agent to some location in the graph so that via its occupation defenders try to optimize a given objective function.

We assume that both teams use the same {\em cooperative path-finding} (CPF) algorithm for reaching temporarily selected targets. Generally, targets can be reassigned multiple times to defender agents in the course of area protection. However, it is assumed that target reassignment does not occur often. After assigning defender agents their target locations they will proceed to their targets via given CPF algorithm. If a target location is reached by a defender agent the agent stops there and continue in occupation of the target location until a new target is assigned to the agent. Attacker agents have their fixed targets in the protected area however they are free to select any temporary target which allows them to move freely in principle.

Our effort is to design a target allocation strategy for the defending team, so the success is measured from the defenders' perspective. The following objective functions can be pursued:

\begin{enumerate}
\item maximize the number of target locations that are not captured by the corresponding attacker
\item maximize the number of target locations that are not captured by the corresponding attacker within a given time limit
\item maximize the sum of distances between the attackers and their corresponding targets
\item minimize the time spent at captured targets
\end{enumerate}

\section{\uppercase{Preliminaries}}

In APP, we model the environment by an undirected unweighted graph $G=(V,E)$. In this work we restrict the instances to 4-connected grid graphs with possible obstacles. The team of attackers and defenders is denoted by $A=\{a_1,\dots, a_m\}$ and $D=\{d_1,\dots d_n\}$, respectively.
Continuous time is divided into discrete time steps. At each time step agents are placed in vertices of the graph so that at most one agent is placed in each vertex. Let $\alpha_t: A \cup D\rightarrow V$ be a uniquely invertible mapping denoting configuration of agents at time step $t$.
Agents can wait or move instantaneously into adjacent vertex between successive time steps to form the next configuration $\alpha_{t+1}$. Abiding by the following movement rules ensures preventing conflicts:
\begin{itemize}
\item An agent can move to an adjacent vertex only if the vertex is empty, or is being left at the same time step by another agent
\item A pair of agents cannot swap across a shared edge
\item No two agents enter the same adjacent vertex at the same time
\end{itemize}

We do not assume any specific order in which agents perform their conflict free actions at each time step. However, our experimental implementation moves all attacking agents prior to moving all defender agents at each time step.

The mapping $\delta^A: A\rightarrow V$ assigns a unique target to each attacker. The task in APP is to move defender agents so that area specified by $\delta^A$ is protected. This task can be equivalently specified as a search for strategy of target assignments for the defender team. That is, we are trying to find an injective mapping $\delta_{t}^D : D\rightarrow V$ which specifies where defender agents should proceed via given path-finding algorithm at time step $t$ as a response to previous attackers movements. The superscripts $A$ and $D$ is sometimes dropped when there is no danger of confusion. Let us note that target reassignment can be done at each time step which is equivalent to full control of movements of defender agents at each time step.

Formally, we state the APP as a decision problem and an optimization problem as follows:

\begin{definition} The decision APP problem:
Given an instance $\Sigma = (G, A, D, \alpha_0, \delta^A)$ of APP, is there a strategy of target allocations $\delta_{t}^D : D\rightarrow V$ such that the team $D$ of defenders is able to prevent agents from the team of attackers from reaching their targets by moving defending agents towards $\delta_{t}^D$.
\end{definition}

In many instances it is not possible to protect all targets. We are therefore also interested in the optimization variant of the APP problem:

\begin{definition} The optimization problem
Given an instance $\Sigma = (G, A, D, \alpha_0, \delta^A)$ of APP, the task is to find a strategy of target allocations $\delta_{t}^D : D\rightarrow V$ such that the team $D$ of defenders minimizes the number of attackers that reach their target by moving defending agents towards $\delta_{t}^D$.
\end{definition}

\section{\uppercase{Theoretical properties}}

APP is a computationally challenging problem as shown in the following analysis. In order to study theoretical complexity of APP we need to consider the decision variant of APP. Many game-like problems are PSPACE-hard, and APP is not an exception. We reduce the known problem of checking validity of Quantified Boolean Formula (QBF) to it.


The technique of reduction of QBF to APP is inspired by a similar reduction of QBF to ACPF from which we borrow several technical steps and lemmas \cite{IvanovaS14}. We describe the reduction from QBF using the following example. Consider a QBF formula in prenex normal form
\begin{multline}
$$
\varphi=\exists x\forall a \exists y\forall b\exists z\forall c \\
(b\vee c\vee x)\wedge(\neg a\vee\neg b\vee y)\wedge          \\
(a\vee \neg x\vee z)\wedge(\neg c\vee \neg y\vee \neg z)
$$
\end{multline}
This formula is reduced to an APP instance depicted in Fig. \ref{fig:proof}. Let $n$ and $m$ be the number of variables and clauses, respectively. The construction contains three types of gadgets.

For an \textbf{existentially} quantified variable $x$ we construct a diamond-shape gadget consisting of two parallel paths of length $m+2$ joining at its two endpoints. 

\begin{figure}[!h]
  \centering
   {\epsfig{file = 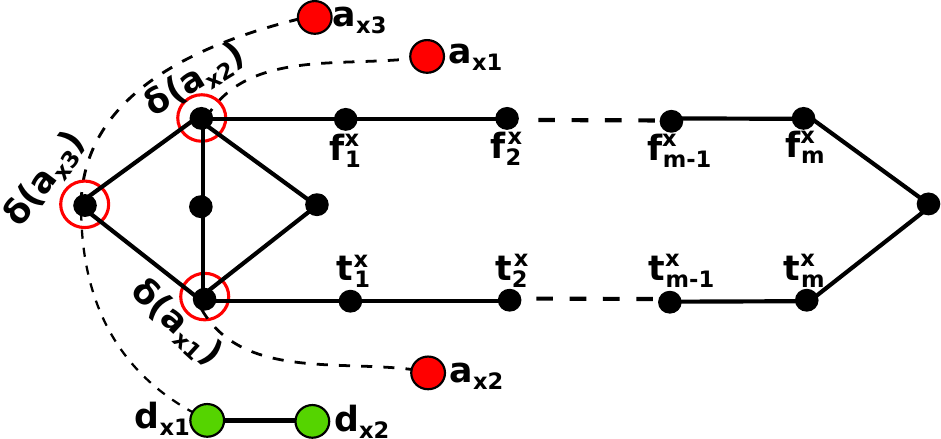, width = 7.0cm}}
  \caption{An existentially quantified variable gadget.  }
  \label{fig:gadget-exists}
 \end{figure}
 
There are 4 paths connected to the diamond at specific vertices as depicted in Fig. \ref{fig:gadget-exists}. The gadget further contains three attackers and two defenders with initial positions at the endpoints of the four joining paths. The vertices in red circles are targets of specified attackers. The only chance for defenders $d_{x1}$ and $d_{x2}$ to prevent attackers $a_{x3}$ and $a_{x1}$ from reaching their targets is to advance towards the diamond and occupy $\delta_A(a_{x3})$ by $d_{x2}$ and either $\delta_A(a_{x1})$ or $\delta_A(a_{x2})$ by $d_{x1}$. 

For every \textbf{universally} quantified variable $a$ there is a similar gadget with a defender $d_{a1}$ and an attacker $a_{a1}$ whose target $\delta^A(a_{a1})$ lies at the leftmost vertex of the diamond structure (see Fig. \ref{fig:gadget-forall}). The defender has to rush to the attacker's target and occupy it, because otherwise the target would be captured by the attacker. 
\begin{figure}[!h]
  \centering
   {\epsfig{file = 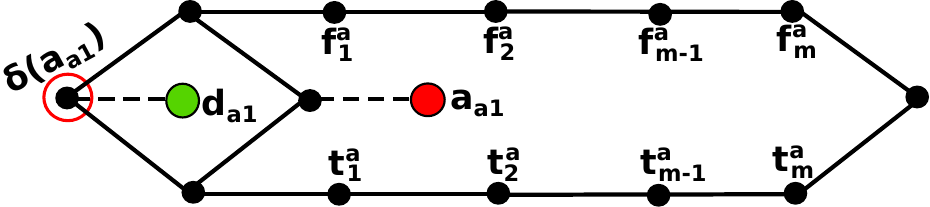, width = 7.0cm}}
  \caption{A universally quantified variable gadget.}
  \label{fig:gadget-forall}
 \end{figure}

Moreover, there is a gadget in two parts for each \textbf{clause} $C$, depicted in Fig. \ref{fig:clause}. It contains a simple path $p$ of length $\lfloor n/2 \rfloor+1$ with a defender $d_C$ placed at one endpoint. This length of $p$ is chosen in order to ensure a correct time of $d_C$'s entering to a variable gadget, so that gradual assignment of truth values is simulated. E. g. if a variable occurring in $C$ stands in the second $\forall\exists$ pair of variables in the prefix (the first and last pair is incomplete), then $p$ is connected to the corresponding variable gadget at its second vertex. The second part of the clause is a path of length $k$, with one endpoint occupied by attacker $a_C$ whose target $\delta^A(a_C)$ is located at the other endpoint. The length $k$ is selected in a way that the target $\delta^A(a_C)$ can be protected if the defender $d_C$ arrives there on time, which can happen only if it uses the shortest path to this target. If $d_C$ is delayed by even one step, the attacker $a_C$ can capture its target. These two parts of the clause gadget are connected through variable gadgets. 

\begin{figure}[!h]
  \centering
   {\epsfig{file = 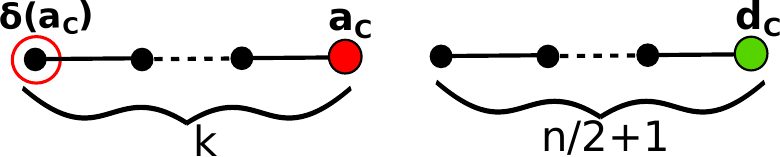, width = 5.0cm}}
  \caption{A clause gadget.}
  \label{fig:clause}
 \end{figure}

The connection by edges and paths between variable and clause gadgets is designed in a way that allows the agents to synchronously enter one of the paths of the relevant variable gadget. A gradual evaluation of variables according to their order in the prefix corresponds to the alternating movement of agents. A defender $d_C$ from clause $C$ moves along the path of its gadget, and every time it has the opportunity to enter some variable gadget, the corresponding variable is already evaluated. 

If there is a literal in $\varphi$ that occurs in multiple clauses, setting its value to true causes satisfaction of all the clauses containing it. This is indicated by a simultaneous entering of affected agents to the relevant path. Each clause defender $d_C$ has its own vertex in each gadget of a variable present in $C$, at which $d_C$ can enter the gadget. This allows a collision-free entering of multiple defenders into one path of the gadget.

\begin{figure}[!h]
  \centering
   {\epsfig{file = 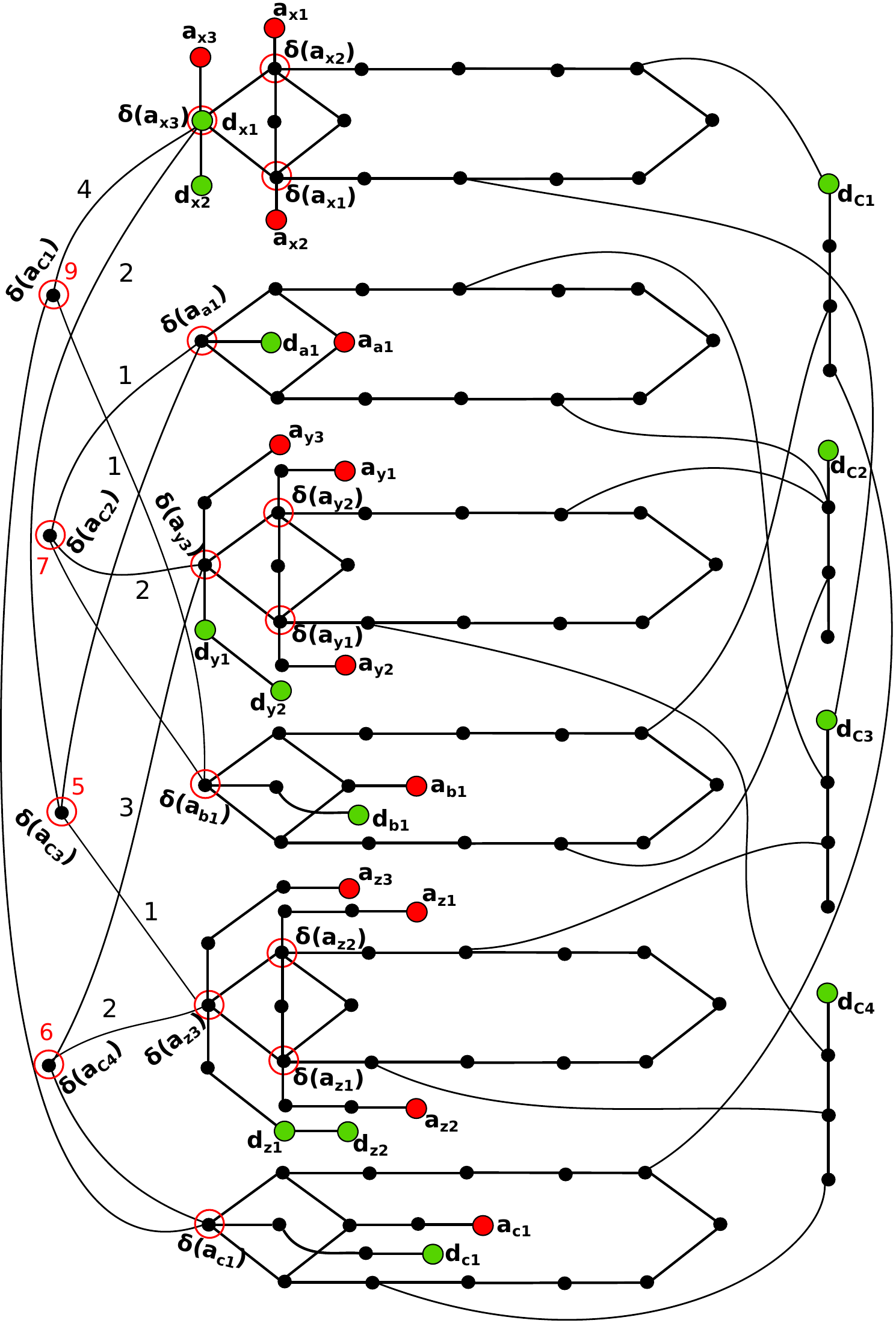, width = 8.0cm}}
  \caption{A reduction from TQBF to APP. Black points represent unoccupied vertices. If two points are connected by a line without any label, it means there is an edge between them. A line with a label $k$ indicates that the two points are connected by a path of $k$ internal vertices. Initial positions of attackers and defenders are represented by red and green nodes, respectively. A red circle around a node means that the node is a target of some attacker. In order to simplify the drawing, we do not fully display the second part of the clause gadget. Instead, there is a red number near the target of a clause gadget that indicates the distance of the attacker aiming to that target. A vertex with an agent is labeled by the agent's name. Labels of targets specify the associated agents. } 
  \label{fig:proof}
 \end{figure}

\begin{theorem}
The decision problem whether there exists a winning strategy for the team of defenders, i.e. whether it is possible to prevent all attackers from reaching their targets, in a given APP instance is PSPACE-hard.
\end{theorem}
\begin{proof}
Suppose $\varphi$ to be valid. To better understand validity of $\varphi$ we can intuitively ensure that variables are assigned gradually according to their order in the prefix. For every choice of value of the next $\forall$-variable there exists a choice of value for the corresponding $\exists$-variable so that eventually the last assignment finishes a satisfying valuation of $\varphi$. The strategy of assigning $\exists$ variables can be mapped to a winning strategy for defenders in the APP instance constructed from $\varphi$. Every satisfying valuation guides the defenders towards vertices resulting in a position where all targets are defended.  Every time  a  variable  is valuated, another agent in the constructed APP instance is ready to enter the upper path, if the variable is evaluated as true, or the lower path, otherwise. When the evaluated variable $x$ is existentially quantified, the defender $d_{x1}$ enters the upper or lower path. In case of universally quantified variable $a$, the entering agent is the attacker $a_{a1}$.  Since the valuation satisfies $\varphi$, every clause $C_j$ has at least one variable $q$ causing the satisfaction of $C_j$. That is modeled by the situation where defenders $d_{q1}$ and $d_{Cj}$ meet each other in one of the diamond's paths, which enables either the defender $d_{q2}$ (in case $q$ is existentially quantified) or $d_{q1}$ (in case $q$ is universally quantified) to advance towards the target $\delta^A(d_{C1})$. The situation for an existentially quantified variable is explained by Fig. \ref{fig:gadget-exists-steps}.

Whenever there exists a winning strategy for the constructed APP instance, the defenders must arrive in all targets on time. This is possible only if variable agents and clause agents meet on one of the paths in a diamond gadget, and only if all defenders use the shortest possible paths. The variable agents' selection of upper or lower paths determines the evaluation of corresponding variables. An advancement of variable and clause defenders that leads to meeting of the defenders at adjacent vertices, and a subsequent protection of targets indicates that the corresponding variable causes satisfaction of the clause. 
\end{proof}

\begin{figure}[!h]
  \centering
   {\epsfig{file = 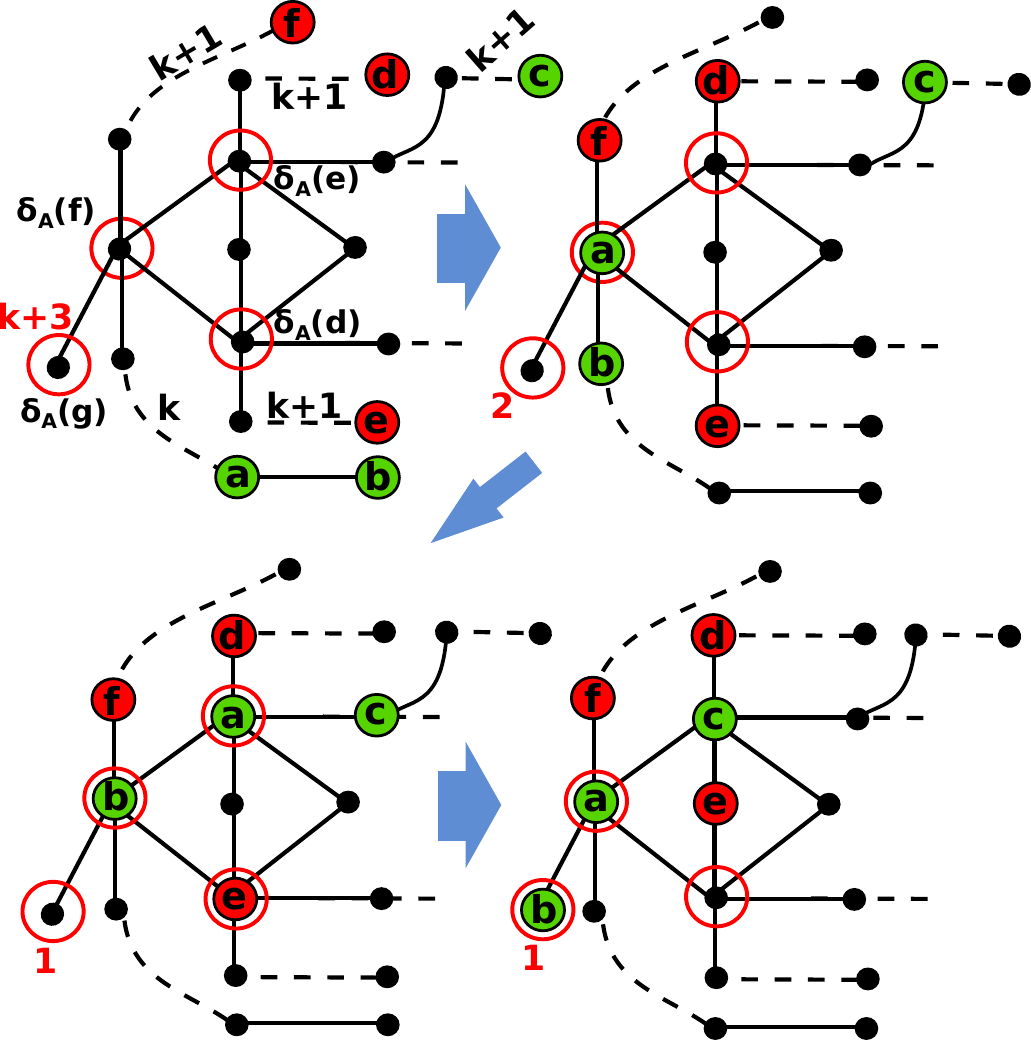, width = 7.0cm}}
  \caption{An example of agents' advancement at an existential variable clause. It is defenders' turn in each of the four figures. The top left figure shows the initial position of the agents. The value $k$ depends on the order of the corresponding variable in the prefix. When the agents reach the positions in the second figure, the corresponding variable is about to be evaluated which is analogous to defender $a$ entering one of the paths, which prevents the attackers $d$ and $e$ from exchanging their position and reaching the targets. If $a$ moves to the upper path, as happens in the third figure, the agent $c$ from a clause gadget has the opportunity to enter the upper path where the two defenders meet. Attacker $e$ can enter the target $\delta^A(d)$, which is nevertheless not its intended goal. Finally, the defenders can protect all targets by a train-like movement resulting in the position in the last figure. Also note the gradual approaching of the undisplayed attacker to the target represented by the red number.      }
  \label{fig:gadget-exists-steps}
 \end{figure}
\section{\uppercase{Target allocation}}

Solving APP in practice is a challenging problem due to its high computational complexity. As already mentioned in the section \ref{sec:introduction}, solving approaches can be divided into two basic categories: \emph{single-stage} and \emph{multi-stage}. In single stage methods, targets are assigned to defenders only once at the beginning, as opposed to multi-stage methods, where the targets can be reassigned any time during the agents' course. Once all defenders are allocated to some targets, they try to get to the desired locations using the LRA* algorithm modified for the environment with adversarial team. This section focuses merely on the single-stage methods.

We describe several simple target allocation strategies and discuss their properties. The first two methods always allocate one defender to one target. The advantage of this approach is that if a defender manages to captures a target, it will never be taken by the attacker. This can be useful in scenarios where the number of defenders is similar to the number of attackers. Unfortunately, such a  strategy would not be very successful in instances where attackers significantly outnumber defenders.

\subsection{Random Allocation}

For the sake of comparison, we consider a strategy, where each defender is allocated to a random target of an attacker. Neither the agent location nor the underlying grid graph structure is exploited.

\subsection{Greedy Allocation}

A greedy strategy is slightly improved approach. The basic variant referred to as \emph{greedy} takes the defenders one by one and allocates it to the closest target. Another variant called \emph{strict greedy} starts with a calculation of distances between every defender and every target, and stores these values in an appropriate data structure. Subsequently, we iteratively select the pair (defender, target) with the shortest distance, and all entries containing the selected attacker and target are removed from the data structure. This is repeated until there are no available defenders or no unassigned targets left.

\subsection{Bottleneck Simulation Allocation}
Simple target allocation strategies do not exploit the structure of underlying graph in any way. Hence natural next step is to occupy by defenders those vertices that would divert attackers from the protected area as much as possible with the help of graph structure. The aim is to successfully defend the targets even with small number of defenders. As our domain are 4-connected grids with obstacles we can take advantage of the obstacles already occurring in the grid and use them as addition to vertices occupied by defenders. Figure \ref{fig:bottleneck} illustrates a grid where the defenders could easily protect the target area even though they are outnumbered by the attackers. Intuitively as seen in the example, hard to overcome obstacle for attacking team would arise if a bottleneck on expected trajectories of attackers is blocked. 
\begin{figure}[!h]
  \centering
   {\epsfig{file = 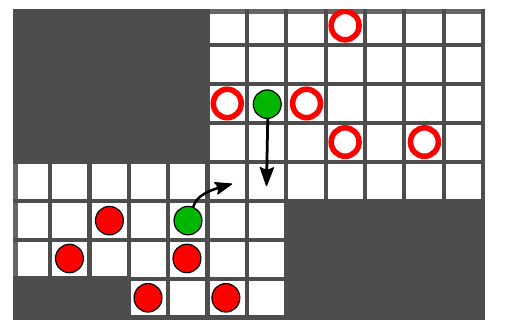, width = 4.0cm}}
  \caption{An example of bottleneck blocking}
  \label{fig:bottleneck}
 \end{figure}
We aim to identify strategic bottlenecks and block them by defenders. The first naive idea assumes that a bottleneck is a gap between two obstacles laying opposite to each other, and it is possible to identify them by an excessive search of the map. It is easy to see that this approach suffers of insufficient robustness. Not only it assumes a map with orthogonal obstacles similar to the one in Fig. \ref{fig:orthogonal_rooms} but also it does not consider whether the bottleneck is actually passed by any attackers. Fig. \ref{fig:bottleneck-compare} shows an example of a bottleneck identifiable by the excessive search (left), and a bottleneck that cannot be discovered by this method (right).

We suggest the following strategy exploiting bottlenecks in the underlying grid.
 \begin{figure}[!h]
    \centering
    \begin{subfigure}[b]{0.13\textwidth}
        \includegraphics[width=\textwidth]{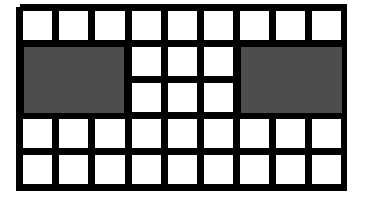}
        \caption{}
        \label{fig:gull}
    \end{subfigure}
    \begin{subfigure}[b]{0.13\textwidth}
        \includegraphics[width=\textwidth]{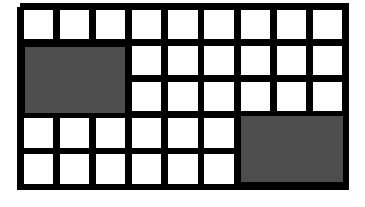}
        \caption{}
        \label{fig:tiger}
    \end{subfigure}
    \vspace{2mm}
    \caption{Examples of bottlenecks}\label{fig:bottleneck-compare}
\end{figure}

In order to discover bottlenecks of general shape, we develop the following simulation strategy. The basic idea is that as attackers move towards the targets, they enter vertices close to a bottleneck more often than other vertices. This observation suggests to simulate the movement of the attackers and find frequently visited vertices. As defenders do not share the knowledge about paths being followed by attackers, frequently visited vertices are determined by a simulation in which paths of attackers are estimated. 

There can be several vertices with the highest frequency of visits, so the final vertex is selected by another criterion. The closer a vertex is to the defenders, the better chance the defenders have to capture it before the attackers pass through it, so in our implementation, we use the distance from an approximate location of defenders in order to select one vertex of maximum frequency.

After obtaining such a frequently visited vertex, we then search its vicinity. If we find out that there is indeed a bottleneck, its vertices are assigned to some defenders as their new targets. Under the assumption that the bottleneck is blocked by defenders, the paths of attackers may substantially change. For that reason we estimate the paths again and find the next frequent vertex of which vicinity is explored. The whole process is repeated until all available defenders are allocated to a target, or until no more bottlenecks are found. The high-level description of this procedure is expressed by Alg. \ref{alg:bottleneck}
\begin{algorithm}
 \KwData{$G=(V,E)$, $D$, $A$}
 \KwResult{Target allocation $\delta^D$}
 $T_{available} = \{\delta (a): a\in A\}$\;
 $D_{available} = D$\;
 $F = \emptyset$\;

 $\delta'_A=$ Random guess of $\delta_A$\;
 \While {$D_{available}\neq \emptyset$}{
  \For{$a\in A$} {
  $p_a= \text{shortestPath}(\alpha_0(a), \delta'_A(a), G, F)$\;
}

  $f(v) = |\{p_a:a\in A \wedge v\in p_a\}|$\;
  $w\in \arg\max_{v\in V} f(v)$\;
  $B=$searchVicinity$(w)$\;
  \eIf{$B\neq \emptyset$} {
  	$D'\subseteq D_{available}, |D'|=|B|$\;
    \text{assignToDefenders}(B, D')\;
    $D_{available}=D_{available}\setminus D'$\;
  	$F = F\cup B$
  }
  {
   \textbf{break} \;
  }
 }
 $\text{assignToDefenders}(T_{available}, D_{available})$\;
 \caption{Bottleneck simulation procedure}
\label{alg:bottleneck}
\end{algorithm}
The input of the algorithm is the graph $G$ and sets $D$ and $A$ of defenders and attackers, respectively. During the initialization phase, we create the set $D_{available}$ of defenders that are not yet allocated to any target. Next, we create the set $F$  of so called forbidden nodes. The following step takes attackers one by one and every time makes a random guess which target is an agent aiming for, resulting in the mapping $\delta'$. The algorithm then iterates while there are available defenders. In each iteration, we construct a shortest path from each agent $a$ between its initial position $\alpha_0(a)$ and its estimated target location $\delta'(a)$. A vertex $w$ from among the vertices contained in the highest number of paths is then selected, and its surroundings is searched for bottlenecks. If a bottleneck is found, the set of vertices $B$ is determined in order to block the bottleneck. The set $D'$ contains a sufficient number of defenders that are allocated to the vertices in $B$. Agents from $D'$ are no longer available, and also vertices from $B$ are now forbidden, so the paths in the following iterations will avoid them. If no bottleneck is found, it is likely that the agents have a lot of freedom for movement and blocking bottlenecks is not a suitable strategy for such an instance. The loop is left and the remaining available agents are assigned to random targets from $T_{available}$. 

The search of the close vicinity of a frequently used vertex $w$ is carried out by an expanding square centered at $w$. We start with distance 1 from $w$ and gradually increase this value\footnote{From the grid perspective, two locations are considered to be in distance 1 from each other if they share at least one point. Hence, a location that does not lie on the edge of the map has 8 neighbors.} up to a certain limit. In every iteration we identify the obstacles in the fringe of the square and keep them together with obstacles discovered in previous iterations. Then we check whether the set of obstacles discovered so far forms more than one connected components. If that is the case, it is likely that we encountered a bottleneck. We then find the shortest path between one connected component of obstacles and the remaining components. This shortest path is believed to be a bottleneck in the map, and its vertices are assigned to the available defenders as their new targets.

In order to discover subsequent bottlenecks in the map, we assume that the previously found bottlenecks are no longer passable. They are marked as forbidden and in the next iteration, the estimated paths will not pass through them. The procedure findShortestPaths returns the shortest path between given source and target, that does not contain any vertices from the set $F$ of forbidden locations. 

In this basic form, the algorithm is prone to finding "false" bottlenecks in instances with an indented map that contains for example blind alleys. It is possible to avoid undesired assigning vertices of false bottlenecks to defenders by running another simulation which excludes these vertices. If the updated paths are unchanged from the previously found ones, it means that blocking of the presumed bottleneck does not affect the attackers movement towards the targets, and so there is no reason to block such a bottleneck.

\begin{figure}
\label{fig:false_bottleneck}
\end{figure}

\section{\uppercase{Experimental evaluation }}
Experimental evaluation is focused on competitive comparison of suggested target allocation strategies with respect to the objective 2. - maximization of the number of locations not captured by attackers within a given time limit.

Our hypothesis is that random strategy would perform as worst since it is completely uninformed. Better results are expected from shortest path and greedy strategy but all these simple strategies are expected to be outperformed by advanced bottleneck strategies.

We implemented all suggested strategies in Java as an experimental prototype. In our testing scenarios we use maps of different structure with various initial configurations of attackers and defenders. Our choice of testing scenarios is focused on comparing different performance of various strategies and discovering what factors have the most significant impact on their success.

As the following sections show, different strategies are successful in different types of instances. It is therefore important to design the instances with a sufficient diversity, in order to capture strengths and weaknesses of individual strategies. 

\subsection{Instance generation and types}

The instances used in the practical experiments are generated using a pseudo random generator, but in a controlled manner. An instance is defined by its map, the ratio $|A|:|D|$ and locations of individual defenders, attackers and their targets. These three entries form an input of the instance generation procedure. Further, we select rectangular areas inside which agents of both teams and the attackers' targets are placed randomly. We use 5 different maps, where the first and most basic map is empty, i. e. does not contain any obstacles. Every next type has a more complicated obstacle structure.
 \begin{figure}[!h]
    \centering
    \begin{subfigure}[b]{0.23\textwidth}
    \centering
        \includegraphics[height = 2.7cm]{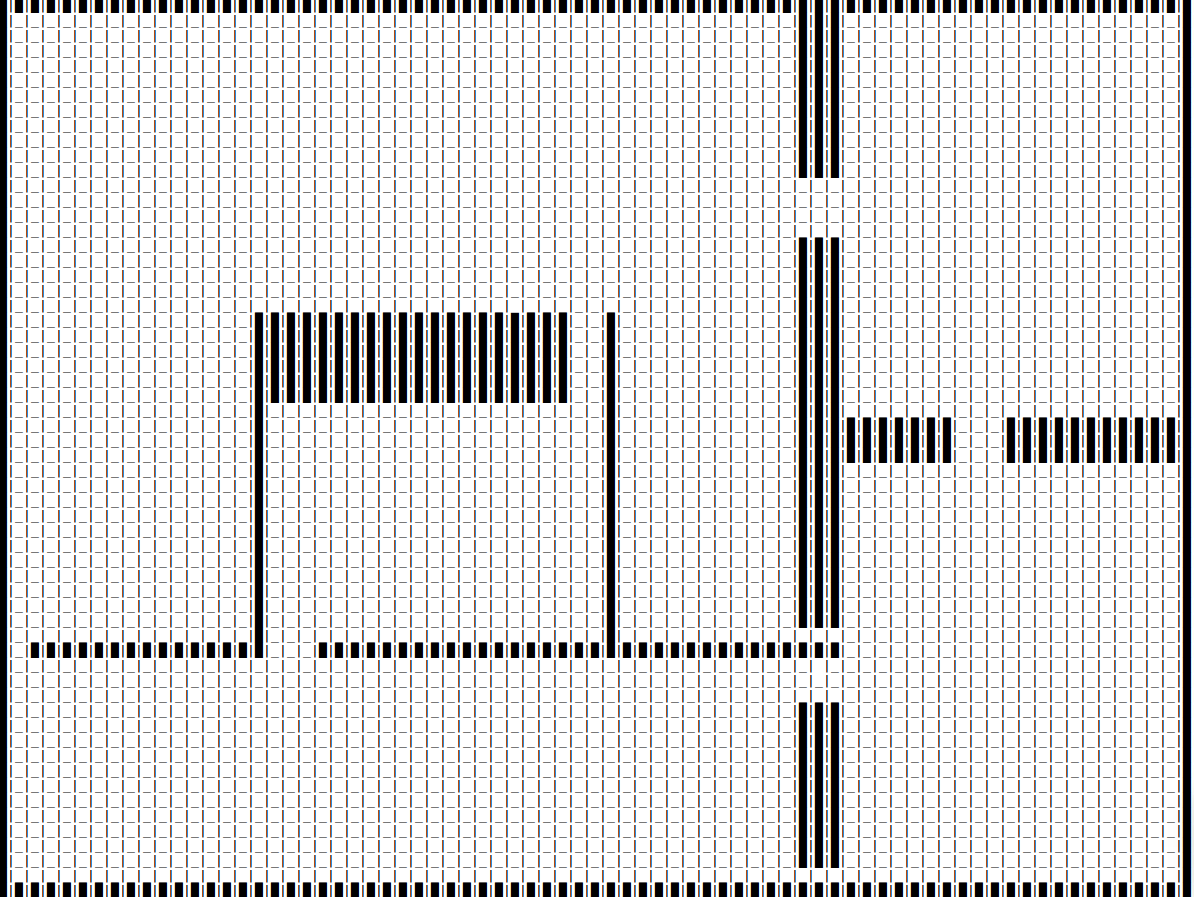}
        \caption{Orthogonal rooms}
        \label{fig:orthogonal_rooms}
    \end{subfigure}
    \begin{subfigure}[b]{0.23\textwidth}
    \centering
        \includegraphics[width = 2.7cm, angle =90]{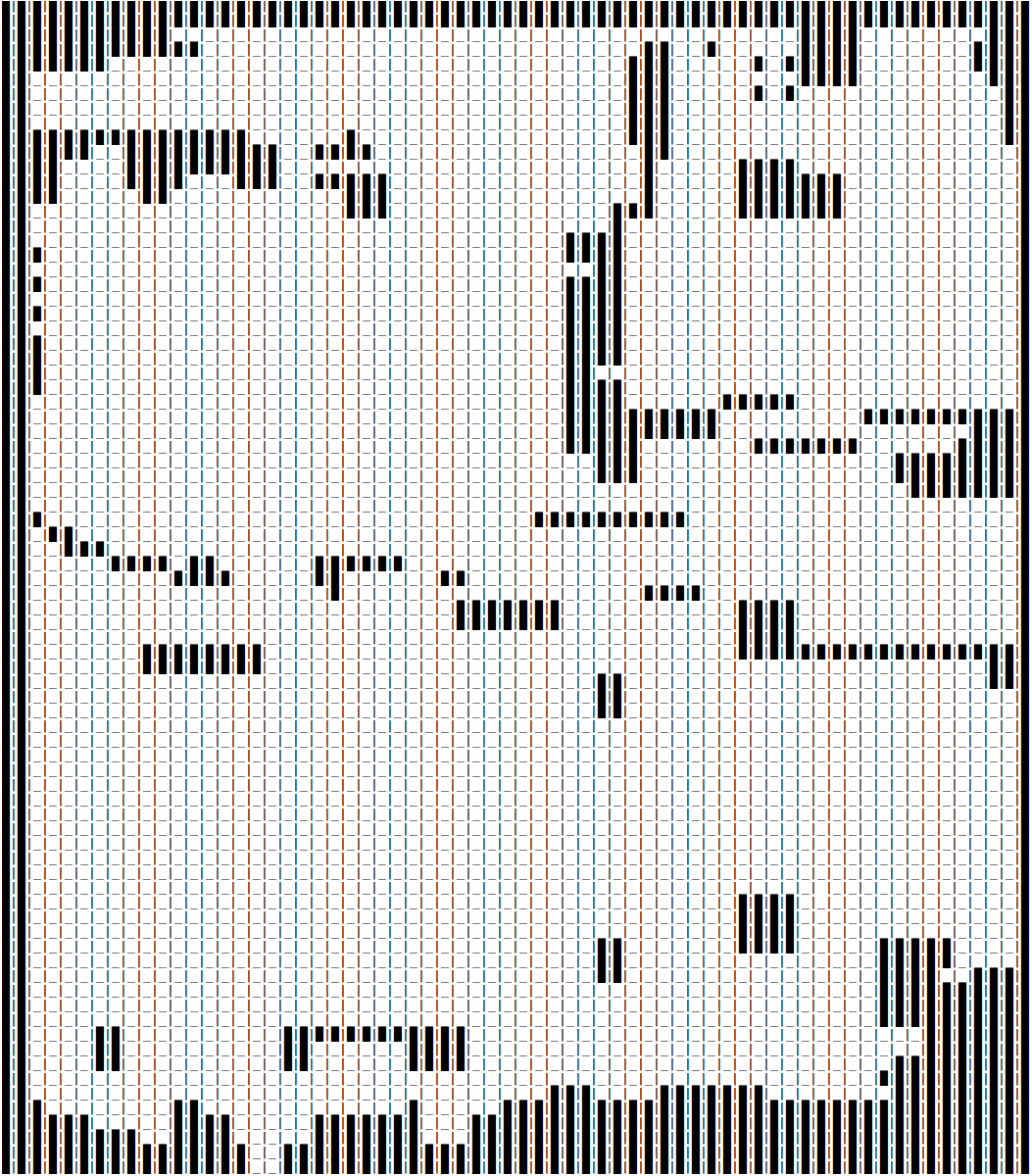}
        \caption{Ruins}
        \label{fig:ruins}
    \end{subfigure}
    \vspace{3mm}
    
    \begin{subfigure}[b]{0.23\textwidth}
    \centering
        \includegraphics[height = 2.60cm]{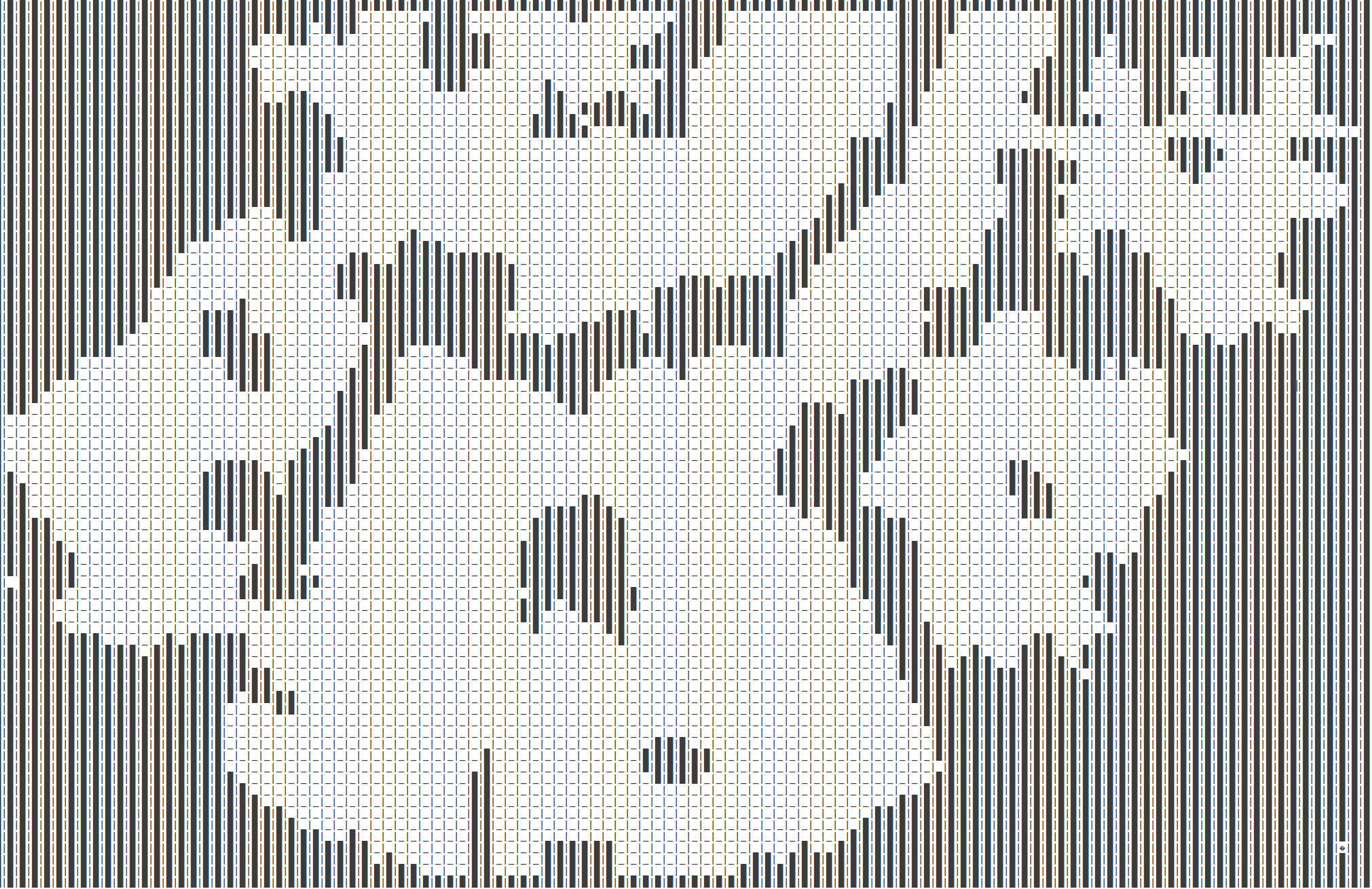}
        \caption{Waterfront}
        \label{fig:waterfront}
    \end{subfigure}
    \begin{subfigure}[b]{0.23\textwidth}
    \centering
        \includegraphics[height = 2.60cm]{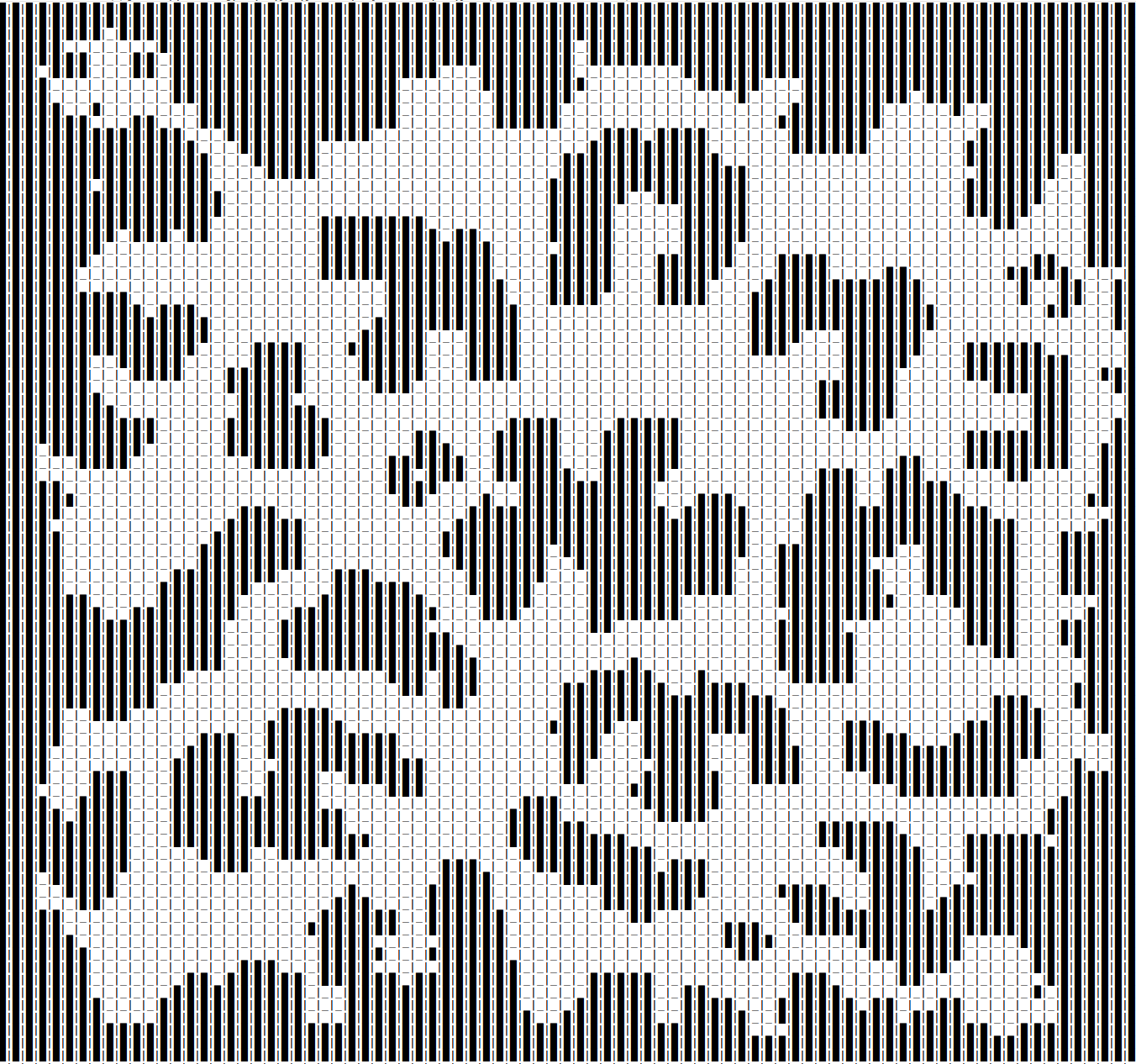}
        \caption{Dark forest}
        \label{fig:dark_forest}
    \end{subfigure}
        \vspace{3mm}
    \caption{Maps}\label{fig:maps}
\end{figure}
\paragraph{Orthogonal rooms.} The first non-empty map depicted in Fig. \ref{fig:orthogonal_rooms} resembles a top view of a house. Doors between the rooms are bottlenecks that often have to be passed by a large number of agents. This map is designed so that any two obstacle pieces that form an obstacle lie on a straight line, which makes them easily detectable. 

\paragraph{Ruins.} This map also resembles a top view of a house, but with damaged walls. Unlike Orthogonal rooms, this map also contains bottlenecks of which walls are dislocated and are therefore harder to recognize, as apparent from the following Fig. \ref{fig:ruins}.

\paragraph{Waterfront.} The next map is even more cluttered. It contains completely irregular obstacles that form disorganized islands connected by bridges, as seen in Fig. \ref{fig:waterfront}. This map is taken from \cite{movingAI}. 
 
\paragraph{Dark forest.} The last and most jumbled map depicted in Fig. \ref{fig:dark_forest} was used in WarCraft and is also taken from \cite{movingAI}. It contains irregular obstacles representing inaccessible forest cover.

In the main set of experiments, each map is populated with agents of 3 different $|D|:|A|$ ratios, namely $1:1$, $1:2$ and and $1:10$, with fixed number of attackers $|A|=100$. Each of these scenarios are further divided into two types reflecting a relative positions of attackers and defenders. The type \emph{overlap} assumes that the rectangular areas for both attackers and defenders have an identical location on the map. On the other hand, the teams in the type \emph{separated} have completely distinct initial areas. Finally, the maximum number of moves of the agents is set to 150 for each team.

Note that the individual instances are never completely fair to both teams. It is therefore impossible to make a conclusion about a success rate of a strategy by comparing its performance on different maps. The comparison should always be made by inspecting the performance in one type of instance, where we can see the relative strength of the studied algorithms.

\subsection{Experiments on simple maps}
The experiments comparing random, shortest path and greedy strategy confirm the expected outcome that the random strategy is always worse then the other two methods. This set of experiments is conducted on instances with an empty map containing 100 agents in each team. Methods that exploit the structure of obstacles in the map are not included because they are not relevant in this settings.

The Fig. \ref{fig:empty-together} depicts scenarios where the rectangular areas of placement of the teams are adjacent to each other and lie on the left side of the empty map. The area that contains targets is initially located so that it partially overlaps both teams areas. It is then gradually dragged to the right end of the map. The distance between the agent areas and the target is represented by the $x$-axis. Each entry in the graph is an average value from 10 runs with different random location of agents and targets in their corresponding areas.
\begin{figure}[!h]
  \centering
   {\epsfig{file = 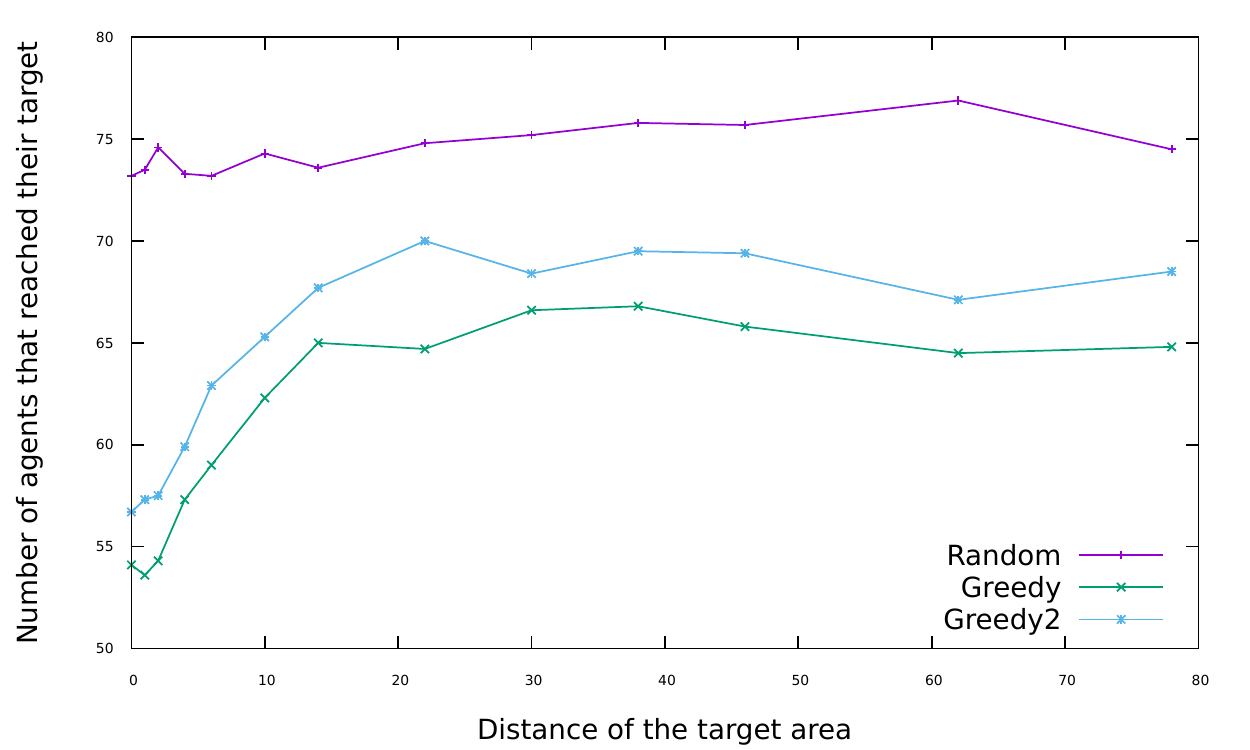, width = 6.0cm}}
  \caption{Empty map, simple strategies, teams together }
  \label{fig:empty-together}
 \end{figure}
A similar experiments with results shown in Fig \ref{fig:empty-opposite} were run an empty map with the targets located in the center, and the both teams were moving away in opposite directions.
\begin{figure}[!h]
  \centering
   {\epsfig{file = 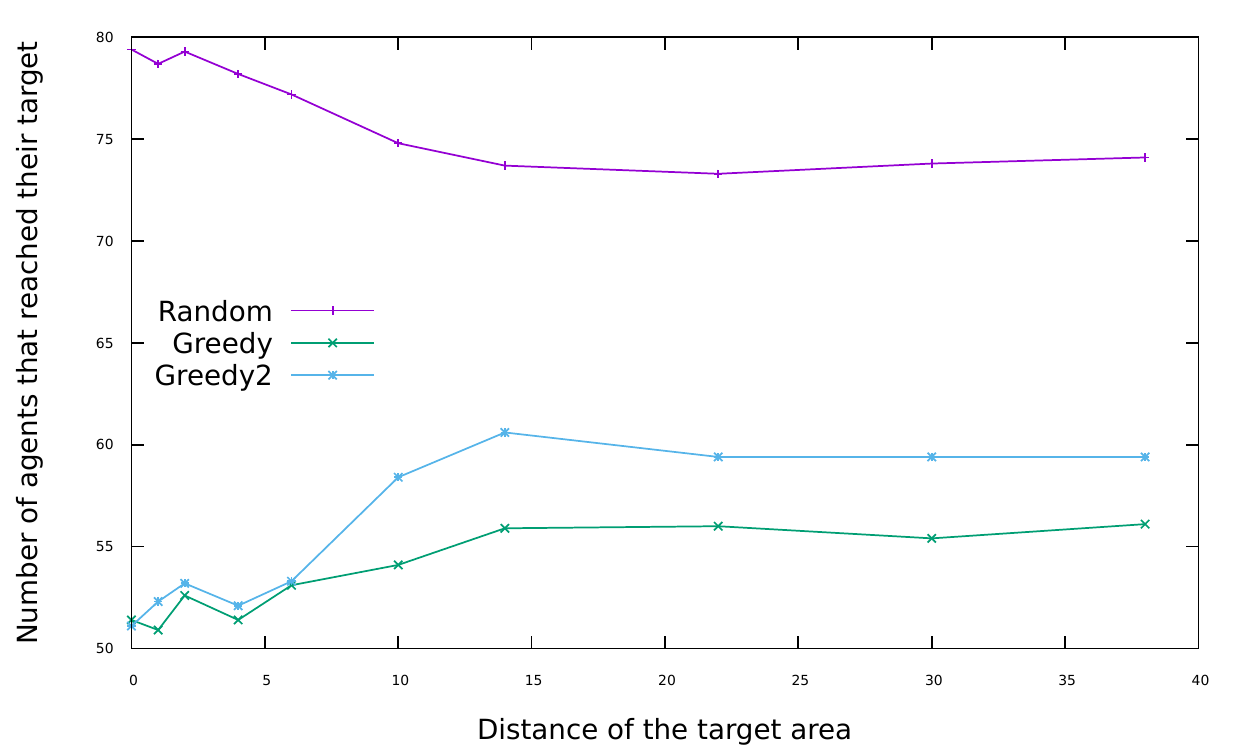, width = 6.0cm}}
  \caption{Empty map, simple strategies, teams opposite }
  \label{fig:empty-opposite}
 \end{figure}
Possibly a less intuitive finding is that the greedy strategy almost always outperforms the strict greedy strategy. By inspecting individual simulation courses, a possible explanation to this behavior is that the strictly greedy strategy successfully captures the closest targets with the closest defenders, but impede defenders in the next frontlines, because they have to bypass them. This delay gives more chances to attackers to reach their targets. 
\begin{figure}[!h]
  \centering
   {\epsfig{file = 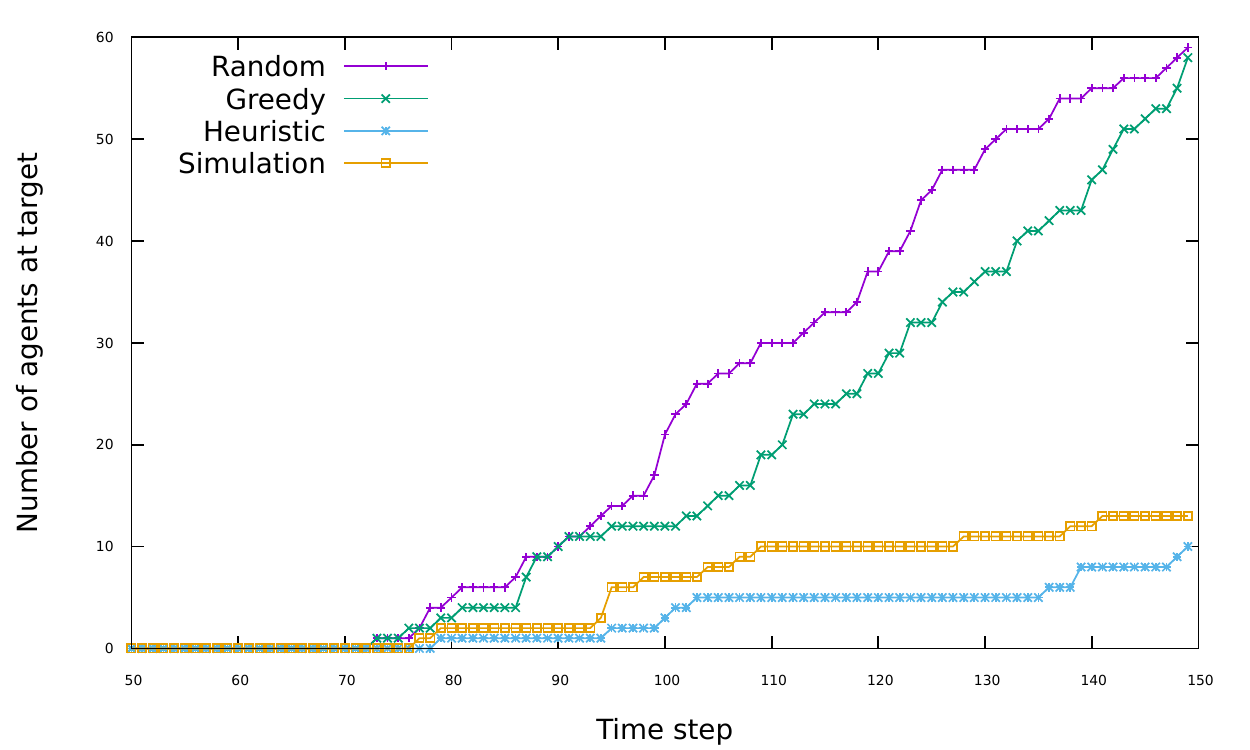, width = 6.0cm}}
  \caption{A single instance on the map Orthogonal rooms.  }
  \label{fig:single_inst}
 \end{figure}
The following set of experiments compares random, greedy, and simulation strategy in different instance settings.  An example of one run, where we compare 4 different strategies is demonstrated in Fig. \ref{fig:single_inst}. We can see how many attackers gradually reach their targets when different target allocation strategies are used in a single instance. The map Orthogonal rooms with overlapping teams was used in this particular example. 

Each entry in Tab. \ref{tab:orthogonal-rooms} is an average number of attackers that reached their targets at the end of the time limit. The average value is calculated for 10 runs in each settings, always with a different random seed. Random and greedy strategies have very similar results in all positions and team ratios. It is apparent and not surprising that with decreasing $|D|:|A|$ ratio, the strength of these strategies decreases. The simulation strategy gives substantially better results in all settings. Also note that in case of overlapping teams, the simulation strategy scores similarly in all $|D|:|A|$ ratios.
 \begin{table}[h]
	\caption{Average number of agents that eventually reached their target in the map Orthogonal rooms}\label{tab:orthogonal-rooms} \centering
\begin{tabular}{llrrr}
   Team\\ position&$|D|:|A|$& RND & GRD& SIM \\
	\hline
  	\multirow{ 3}{*}{Overlapped} & 1:1 & 40.4 & 49.2 & 21.0 \\
  							  & 1:2 & 56.7 & 56.5 & 20.8 \\
  							  & 1:10 & 67.8 & 64.7 & 24.7 \\ \hline
   \multirow{ 3}{*}{Separated} & 1:1 & 39.0 & 40.7 & 10.3 \\
  							  & 1:2 & 57.7 & 50.1 & 13.3 \\
   							  & 1:10 & 78.5 & 69.9 & 30.2 \\
  \hline
\end{tabular}
\end{table}
Tab. \ref{tab:ruins} contains results of an analogous experiment conducted on the map Ruins. The random strategy performs well with in instances with hight number of attackers. The dominance of the simulation strategy is apparent here as well.
 \begin{table}[h]
	\caption{Average number of agents that eventually reached their target in the map Ruins. }\label{tab:ruins} \centering
\begin{tabular}{llrrr}
   Team\\ position&$|D|:|A|$& RND & GRD & SIM \\
	\hline
  	\multirow{ 3}{*}{Overlapped} & 1:1 & 36.8 & 49.4 & 17.7 \\
  							  & 1:2 & 80.0 & 63.5 & 33.0 \\
  							  & 1:10 & 92.5 & 88.9 & 58.2 \\ \hline
   \multirow{ 3}{*}{Separated} & 1:1 & 9.5 & 33.6 & 11.8 \\
  							  & 1:2 & 47.6 & 34.4 & 11.8 \\
   							  & 1:10 & 85.6 & 85.9 &14.7 \\
  \hline
\end{tabular}
\end{table}
\subsection{Experiments on complex maps}
Maps Waterfront and Dark forest contain very irregular obstacles and many bottlenecks, and are therefore very challenging environments for all strategies. In the Dark forest map, random and greedy methods are more suitable than the simulation strategy in instances with equal team sizes, as oppose to the scenarios with lower number of defenders, where the bottleneck simulation strategy clearly leads. In the separated scenario, the simulation strategy is even worse in all tested ratios (see Tab. \ref{tab:waterfront} and Tab. \ref{tab:darkforest}). This behavior can be explained by the fact that occupying all relevant bottlenecks in such a complex map is harder than occupying targets in the protected area. In contrast, bottlenecks in the Waterfront map have more favourable structure, so that those relevant for the area protection can be occupied more easily. 
 \begin{table}[h]
	\caption{Average number of agents that eventually reached their target in the map Waterfront}\label{tab:waterfront} \centering
\begin{tabular}{llrrr}
   Team\\ position&$|D|:|A|$& RND & GRD & SIM \\
	\hline
  	\multirow{ 3}{*}{Overlapped} & 1:1 & 32.0 & 41.7 & 37.1 \\
  							  & 1:2 & 60.6 & 63.8 & 39.8 \\
  							  & 1:10 & 77.8 & 72.9 & 51.7 \\ \hline
   \multirow{ 3}{*}{Separated} & 1:1 & 15.8 & 19.3 & 10.7 \\
  							  & 1:2 & 46.4 & 37.6 &  9.8 \\
   							  & 1:10 & 75.3 & 65.5 & 14.9 \\
  \hline
\end{tabular}
\end{table}
 \begin{table}[h]
	\caption{Average number of agents that eventually reached their target in the map Dark forest}\label{tab:darkforest} \centering
\begin{tabular}{llrrr}
   Team\\ position&$|D|:|A|$& RND & GRD & SIM \\
	\hline
   \multirow{ 3}{*}{Overlapped} & 1:1 & 21.6 & 37.9 & 48.8 \\
  							  & 1:2 & 53.7 & 42.6& 37.8 \\
   							  & 1:10 & 60.9 & 51.9& 38.4 \\
  \hline
\multirow{ 3}{*}{Separated} & 1:1 & 35.3 & 35.9 & 61.5 \\
  							  & 1:2 & 40.6 & 41.3 & 59.6 \\
  							  & 1:10 & 65.1 & 67.0 & 66.0 \\ 
    \hline

\end{tabular}
\end{table}
\section*{\uppercase{Concluding Remarks}}
We have shown the lower bound for computational complexity of the APP problem, namely that it is PSPACE-hard. Theoretical study of ACPF \cite{IvanovaS14} showing its membership in EXPTIME suggests that the same upper bound holds for APP but it is still an open question if APP is in PSPACE. In addition to complexity study we designed several practical algorithms for APP under the assumption of single-stage vertex allocation. Performed experimental evaluation indicates that our {\em bottleneck simulation} algorithm is efficient even in case when defenders are outnumbered by attacking agents. Surprisingly, our simple {\em random} and {\em greedy} algorithms turned out to successfully block attacking agents provided there are enough defenders.

For future work we plan to design and evaluate algorithms under the assumption of multi-stage vertex allocation. As presented algorithms have multiple parameters we also aim on optimization of these parameters.


\bibliographystyle{apalike}
{\small
\bibliography{protective-agents}}

\vfill
\end{document}